\documentclass[11pt]{article}

\usepackage[final]{acl}

\usepackage{times}
\usepackage{latexsym}

\usepackage[T1]{fontenc}
\usepackage[utf8]{inputenc}
\usepackage{microtype}
\usepackage{inconsolata}

\usepackage{graphicx}
\graphicspath{{../}{./}}

\usepackage{amsmath}
\usepackage{amssymb}
\usepackage{mathtools}
\usepackage{amsthm}

\usepackage{multirow}
\usepackage{makecell}
\usepackage[svgnames]{xcolor}
\definecolor{lightyellow}{RGB}{255,255,224} 
\definecolor{abel}{RGB}{25,255,224}
\usepackage[framemethod=TikZ]{mdframed}
\usepackage{graphicx}
\usepackage{listings}
\usepackage{wrapfig}
\usepackage{amsmath,amssymb,amsthm}

\usepackage[table]{xcolor}
\definecolor{lightblue}{RGB}{230,245,255}

\theoremstyle{plain}
\newtheorem{theorem}{Theorem}[section]

\newtheorem{lemma}[theorem]{Lemma}

\theoremstyle{definition}

\theoremstyle{remark}

\usepackage{microtype}
\usepackage{graphicx}
\usepackage{subcaption}
\usepackage{booktabs}

\newcommand{\trainable}[1]{\textcolor{DarkGreen}{#1}}
\title{Bi-Phase Training: Learning Efficiently From Scratch}

\author{Abel Gurung \\
  Purdue University \\
  \texttt{gurung1@purdue.edu} \\\And
  Joseph Campbell \\
  Purdue University \\
\texttt{joecamp@purdue.edu} \\}
\begin{document}
\maketitle
{\renewcommand{\thefootnote}{}\footnotetext{Accepted to Findings of AACL-IJCNLP 2026.}}
\begin{abstract}
Pre-training foundational models is computationally expensive as models scale. In this work, we introduce Bi-Phase Training (BPT), a parameter-efficient pre-training method that preserves the high-rank learning dynamics observed in full-parameter optimization with substantially reduced trainable parameters. BPT parameterizes each weight update as a composition of constrained high-rank left/right transformations via trainable diagonal matrices and low-rank adapters whose updates are periodically merged into the frozen base weights, resulting in a compounding expansion of accessible directions during training. Across language-model pre-training from scratch (Qwen2.5 models from 100M to 1.5B parameters and OLMo-2-1B) on FineWeb-Edu (10B and 100B tokens), BPT achieves validation perplexity close to full-parameter baselines while using far fewer trainable parameters—e.g., at the 1.5B scale, BPT uses 66\% fewer trainable parameters with only a small validation-perplexity gap. Additionally, a comprehensive evaluation across diverse downstream tasks shows that BPT preserves much of the full-parameterized model's downstream performance.

\end{abstract}

\section{Introduction}
\label{sec:intro}

Foundational models have demonstrated impressive general-purpose performance \cite{deepseekai2025deepseekv3technicalreport, grattafiori2024llama3herdmodels}. These models usually consist of billions of parameters, are trained on massive datasets, and have become the standard in modern deep learning. Empirical studies such as scaling laws have demonstrated that increasing model size generally leads to lower training loss \cite{kaplan2020scalinglawsneurallanguage, hoffmann2022trainingcomputeoptimallargelanguage}. However, training larger models requires significantly more compute, highlighting the need for more compute-efficient pre-training methods.

\begin{figure}[t]
  \centering
  \includegraphics[width=\linewidth,keepaspectratio]{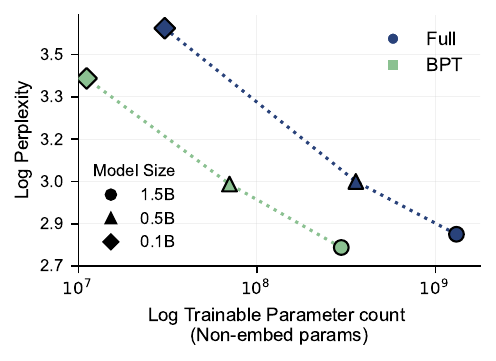}
  \caption{Evaluation log-perplexity vs. log trainable parameters (non-embedding) across model sizes. BPT approaches the fully-parameterized baseline while using far fewer trainable parameters; on the 1.5B model it achieves comparable eval loss with 66\% fewer trainable parameters}
  \label{fig:log_log_scale}
\end{figure}

During training, a given trainable parameter incurs a fixed memory overhead: the parameter itself ($n$), its gradient ($n$), and two optimizer states ($2n$) when using Adam family of optimizers \cite{kingma2017adammethodstochasticoptimization, loshchilov2019decoupledweightdecayregularization}. While total training memory depends on several factors such as batch size and activation footprint \cite{chen2016trainingdeepnetssublinear}, this $4n$ overhead remains fixed per trainable parameter. Distinguishing between non-trainable and trainable parameters for pre-training is important, and the total number of parameters can be interpreted as the model's representational capacity. Our central claim is that effective pre-training does not require optimizing every parameter. Instead, models can retain much of the learning behavior of full training while updating only a small subset of parameters, provided that the resulting updates preserve the essential training dynamics. We first analyze the general learning dynamics of full-parameter training in terms of the rank of the update for SGD and Adam optimizers, showing that in full-parameter training, updates are high-rank, and verify this behavior on a simple model. Then, to preserve the learning dynamics of high-rank while being parameter-efficient, we introduce Bi-Phase Training (BPT), a novel parameter-efficient method for pre-training foundation models that updates a weight matrix $\mathrm{W}$ by simultaneously applying constrained high-rank and low-rank updates. This work directly builds on prior works such as LoRA \cite{hu2021loralowrankadaptationlarge}, ReLoRA \cite{lialin2023relorahighranktraininglowrank}, and HyperAdapt \cite{gurung2025hyperadaptsimplehighrankadaptation} and combines them in pre-training setting to achieve high-rank updates at every optimization step.

The claim that a high-rank update is necessary is consistent with prior evidence. ReLoRA \cite{lialin2023relorahighranktraininglowrank} shows that allowing a brief period of full-parameter training before switching to low-rank adaptation improves optimization, suggesting early learning benefits from richer (higher-rank) update directions \cite{lialin2023relorahighranktraininglowrank}. Complementarily, systematic ablations that increase LoRA rank $r$ toward full-rank updates report that doing this matches full training performance when the model is trained from scratch, reinforcing the link between update rank and training performance \cite{huh2024trainingneuralnetworksscratch}.

Empirically, we demonstrate the effectiveness of our method by pre-training three language models of varying sizes, showing consistent performance across scales. We also perform a comprehensive downstream evaluation of the trained models to show that there is minimal performance loss when using fewer trainable parameters than when using all trainable parameters. Our main contributions are as follows:

\begin{itemize}
  \item We introduce Bi-Phase Training (BPT) to significantly reduce the number of trainable parameters required during pre-training.
  \item We provide a theoretical upper-bound to the rank of the update induced by BPT.
  \item We comprehensively test the method on downstream evaluation tasks and show that there is minimal performance difference between the fully parameterized training and BPT.
\end{itemize}

\section{Background}
\label{background}

\paragraph{Full Parameterized Training}

\begin{figure*}[t]
  \centering
  \includegraphics[width=1\linewidth]{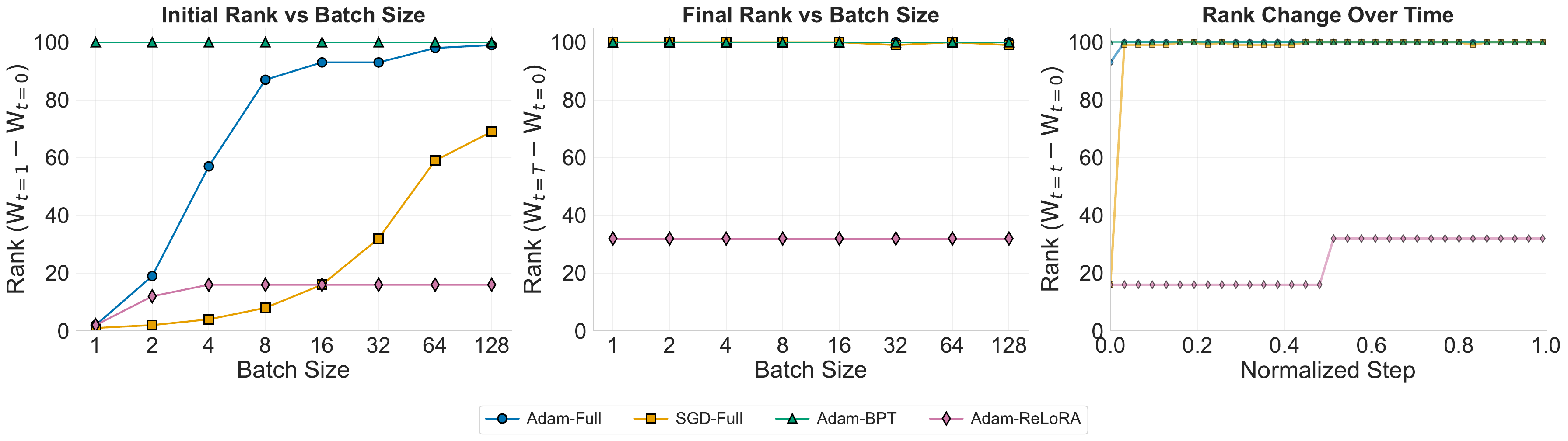}
  \caption{Update-rank trend across batch sizes and methods. {Left:} rank change after the first update step versus batch size, showing scaling effects for SGD and consistently high-rank updates for Adam-full and BPT. {Middle:} rank at the end of training ($t=T$) versus batch size, where full-parameter methods and BPT reach near-full rank even from small batches, while ReLoRA remains rank-limited. {Right:} rank evolution over normalized training steps for batch size 16, comparing Adam-full, SGD-full, BPT, and ReLoRA.}
  \label{fig:rank_plot_adam}
\end{figure*}

During neural network training, we aim to find weight matrices $\mathrm W \in \mathbb{R}^{n \times m}$ that minimizes a loss function, $L$. The matrix is updated iteratively using gradient signals from backpropagation. The rank of the update $\Delta \mathrm{W}$ indicates the number of independent update directions applied per step to the weight matrix during optimization. Formally, we want to optimize the weight matrix $\mathrm W$ using $\Delta \mathrm{W}$:
$
\mathrm{W'} = \mathrm{W} + \Delta \mathrm{W}
$
where $\mathrm{W'}$ is the optimized weight matrix. For a single training example with an input vector $x \in \mathbb{R}^m$ and a corresponding layer output $y \in \mathbb{R}^n$, the gradient of the loss with respect to the weight matrix $G$ is the outer product of the upstream gradient and the input vector's transpose:
$$
G = \frac{\partial L}{\partial \mathrm{W}} = \left(\frac{\partial L}{\partial y}\right) x^T
$$
Since this is an outer product of two vectors, its rank is at most one, meaning that one training data point gives rank-one information. In stochastic gradient descent (SGD), where only the gradient of the matrix is used for the update, $\Delta \mathrm W$, the rank of $\Delta \mathrm W$ for a single example is at most one. For a mini-batch of $B$ training examples, we can stack the inputs into a matrix $\mathrm X \in \mathbb{R}^{m \times B}$ and the corresponding upstream gradients into a matrix $\Lambda \in \mathbb{R}^{n \times B}$. The gradient for the entire mini-batch is then given by the matrix product:
$$
G = \frac{1}{B} \Lambda \mathrm X^T = \frac{1}{B} \sum_{i=1}^B \lambda_i x_i^T
$$
The rank of a matrix product is less than or equal to the minimum rank of its factors. Since the rank of $\Lambda$ is at most $\min(n, B)$ and the rank of $\mathrm X^T$ is at most $\min(B, m)$, the rank of the resulting gradient matrix is bounded by:
$$
\operatorname{rank}(G) \le \min(\operatorname{rank}(\Lambda), \operatorname{rank}(\mathrm X^T))
$$
Without loss of generality, let $B \le \min(n,m)$. For the update $\Delta \mathrm W$ in SGD, the $\operatorname{rank}(\Delta \mathrm W)$ is at most $B$. This means that the update rank's upper-bound grows linearly with the batch size for SGD. Additionally, adding momentum or using Adam has a different upper-bound for the update rank even at smaller batch size, refer to \autoref{sec:adams_update_rank} for more on momentum and Adam's rank.

\paragraph{Parameter Efficient Training} Full parameterized training can induce high-rank updates. However, this requires a large number of trainable parameters which is computationally expensive. In this section, we introduce parameter-efficient methods to train models while reducing the number of trainable parameters. To reduce the number of trainable parameters in $\mathrm W$, we want to express $\Delta \mathrm W$ with minimal number of trainable parameters while having high-rank such that the trainable parameter reduction does not alter its learning dynamics.

In LoRA \cite{hu2021loralowrankadaptationlarge}, $\Delta \mathrm{W}$ is defined as the product of two low-rank matrices $\mathrm{B}\in \mathbb{R}^{n \times r}$ and $\mathrm{A} \in \mathbb{R}^{r \times m}$ such that $\Delta \mathrm{W} = \alpha \mathrm{BA}$, where $\alpha$ is a scaling factor. LoRA trains the low-rank matrices, $\mathrm{B}$ and $\mathrm{A}$, which contain significantly fewer trainable parameters than $\mathrm W$. LoRA yields strong results for fine-tuning a pre-trained weight matrix, however, the objective landscape for pre-training from scratch is too complex to be modeled by a single $r$-rank subspace alone \cite{lialin2023relorahighranktraininglowrank}.
Previous works \cite{huh2024trainingneuralnetworksscratch} have empirically shown that for training models from scratch with LoRA with rank $r = \min(n,m)$ equaling the maximum rank possible of the matrix can match the performance of full training from scratch. Since the $\Delta \mathrm{W}$ rank in LoRA is always bounded by $r$, it can never utilize the full rank potential of $\mathrm W$. To address this limitation, ReLoRA \cite{lialin2023relorahighranktraininglowrank} extends LoRA by accumulating low-rank updates over multiple steps:
\begin{align*}
  \mathrm{\Delta W} = \alpha \sum_{n=1}^N \mathrm{B}_n \mathrm{A}_n
\end{align*}
where $\mathrm{B}$ and $\mathrm{A}$ are low-rank matrices. During training, $\mathrm{B}$ and $\mathrm{A}$ are optimized, and their low-rank matrices are periodically merged back into the original weight matrix $\mathrm{W}$. Following each merge, \(\mathrm{B}\) and \(\mathrm{A}\) are reinitialized, allowing the accumulated update \(\Delta \mathrm{W}\) to surpass the rank constraint \(r\) after multiple merge cycles. Despite this improvement, the rank of updates at \textbf{any single step} remains confined to a low-rank subspace. Also, for the duration of $\Delta t$ between two merge cycles, the update is restricted to only a rank-$r$ update. This limits its ability to learn as efficiently as a fully parameterized model, as the update rank for a full parameterized model for any given timestep can be high-rank and results in weaker performance compared to full parameter pre-training ~\cite{lialin2023relorahighranktraininglowrank}. Additionally, merging less frequently gives better performance empirically, since the Adam EMAs are better estimates as training progresses \cite{lialin2023relorahighranktraininglowrank,huh2024trainingneuralnetworksscratch}.

\begin{figure*}[h!]
  \centering
  \includegraphics[width=1\linewidth]{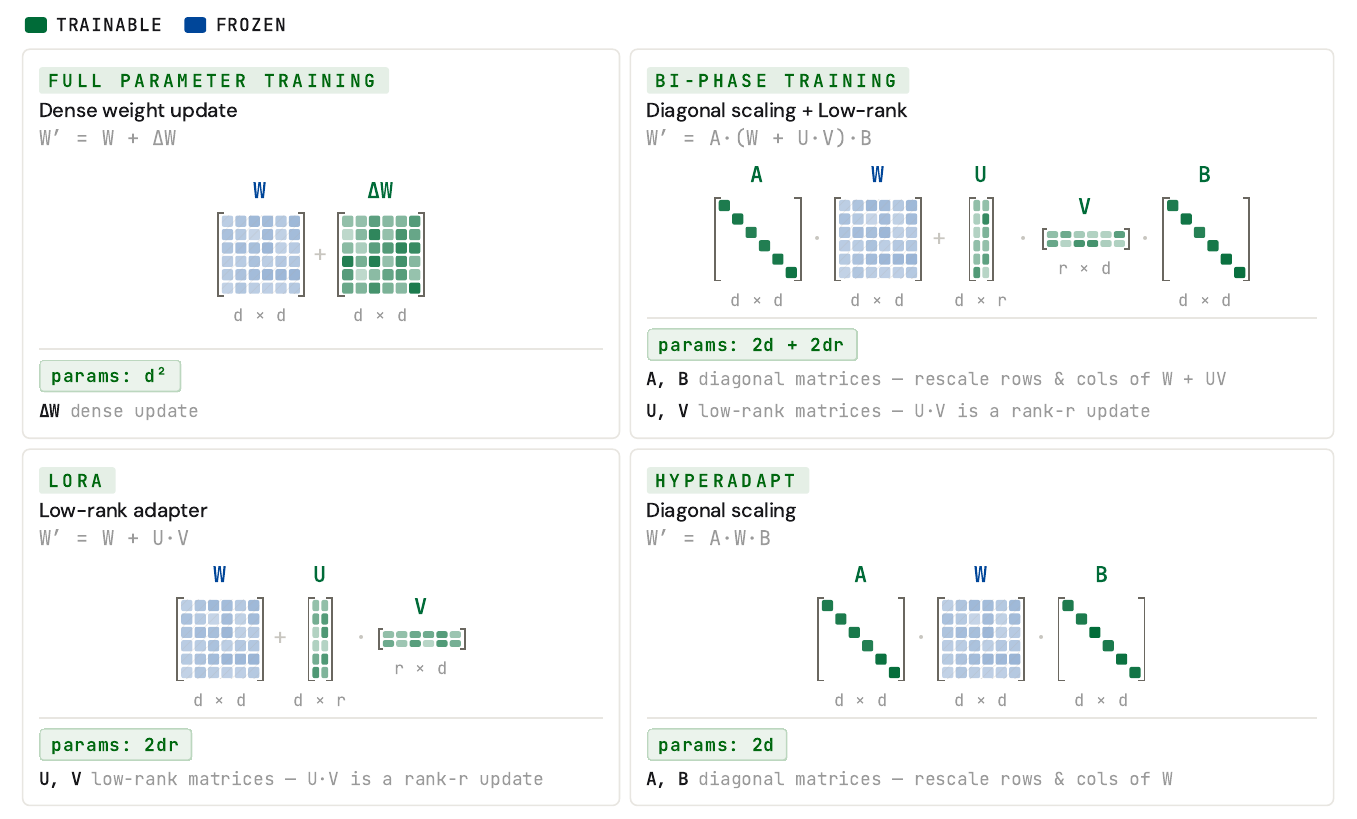}
  \caption{Overview of parameter-efficient training techniques. Blue represents non-trainable parameters, and green represents trainable parameters. $\mathrm{UV}$ is the product of two low-rank matrices of rank-$r$, and $\mathrm{A}$ and $\mathrm{B}$ are diagonal matrices.}
  \label{fig:active_params_overview}
\end{figure*}

In a related work, HyperAdapt~\cite{gurung2025hyperadaptsimplehighrankadaptation} proposed a parameter-efficient high-rank adaptation method by defining the update as: $\Delta \mathrm{W} = \mathrm{A}\mathrm{W}\mathrm{B} - \mathrm{W}$
where \(\mathrm{A} \in \mathbb{R}^{n \times n}\) and \(\mathrm{B} \in \mathbb{R}^{m \times m}\) are trainable diagonal matrices, and \(\mathrm{W}\) is non-trainable parameter. Since both trainable matrices are diagonal, only ${n} + {m}$ number of parameters are needed to train them. Here the rank of $\Delta \mathrm{W}$ is upper-bounded by the $2\cdot \mathrm{rank(W)}$. However, this method assumes that pre-trained weight matrices already contain relevant orthogonal directions which can be scaled relevant to the downstream fine-tuning task, strictly limiting its application to fine-tuning pre-trained models.

\textbf{Empirical rank.} We train a three-layer ReLU MLP on CIFAR-10 \cite{torralba2008tinyimages} and track the second-layer update rank, $\mathrm{rank}(W_t-W_0)$, across batch sizes and training steps; results are shown in \autoref{fig:rank_plot_adam}. After one update, Adam-full and BPT produce consistently high-rank updates across batch sizes, while SGD-full follows its batch-size-dependent rank bound and ReLoRA remains rank-limited by construction. This suggests that BPT’s diagonal matrices can efficiently modulate the rank-one directions of the base weight matrix. By the end of training, all methods except ReLoRA reach near-full-rank updates regardless of batch size, with BPT closely matching full-parameterized training. For batch size 16, SGD-full begins with low-rank updates but reaches full rank within a few steps, indicating that high-rank updates emerge quickly during full training.

\section{Our Method}

Since a fully parameterized model can express high-rank updates at any step $t$, as shown in \autoref{fig:rank_plot_adam}, we seek an update $\Delta \mathrm{W}$ whose rank is comparably high while minimizing total number of trainable parameter to express $\Delta \mathrm{W}$. To do this, we introduce Bi-Phase Training (BPT), a parameter-efficient method designed to induce high-rank updates within individual weight matrices using diagonal matrices and build low-rank subspace using low-rank matrices.

\paragraph{Bi-Phase Training}

Our method conceptualizes the optimization of a weight matrix $\mathrm W \in \mathbb{R}^{n \times m}$ not as a direct update to all $n \times m$ parameters, but as an update structure $\Delta\mathrm {W} $ that combines transformations applied through efficient high-rank diagonal matrices with exploration guided by low-rank matrices. This allows us to approximate the expressiveness of full parameter updates while drastically reducing the trainable parameter count. Formally, for a weight matrix $\mathrm W_0$, the update $\Delta \mathrm  W$ is defined as:

\begin{equation}
  \label{eq:modified}
  \Delta \mathrm{W} = \mathrm{A}\left(\mathrm{W_0} + \sum_{t=1}^{T}\mathrm{U}\mathrm{V}\right)\mathrm{B} - \mathrm{W_0}
\end{equation}
where \(\mathrm{A} \in \mathbb{R}^{n \times n}\) and \(\mathrm{B} \in \mathbb{R}^{m \times m}\) are diagonal trainable matrices, while \(\mathrm{U} \in \mathbb{R}^{n \times r}\) and \(\mathrm{V} \in \mathbb{R}^{r \times m}\) are low-rank trainable matrices. Importantly, $\mathrm{W_0}$ is a non-trainable parameter during training.
At a given time $t$, we only keep a pair of $\mathrm{U}$ and $\mathrm{V}$ trainable and periodically merge $\mathrm{U}$ and $\mathrm{V}$ to $\mathrm W $. So that means for a given step $t$, we have:
\begin{equation}
  \label{eq:time_t_step}
  \Delta \mathrm{W} = \mathrm{A}\left(\mathrm{W_0} + \mathrm{U}\mathrm{V}\right)\mathrm{B} - \mathrm{W_0}
\end{equation}
At any given time $t$, the update in \autoref{eq:time_t_step} is not restricted to rank $r$ in the same way as a pure low-rank adapter. The lemma below provides an expressivity upper bound, $\operatorname{rank}(\Delta \mathrm{W}) \le \min(n,m,2R+r)$, showing that BPT's update rank is not inherently limited by the rank of $\mathrm{U}\mathrm{V}$. Intuitively, the low-rank matrices $\mathrm{U}$ and $\mathrm{V}$ find and accumulate relevant orthogonal directions within a lower-dimensional subspace, while the diagonal matrices $\mathrm{A}$ and $\mathrm{B}$ scale relevant orthogonal directions, thus effectively guiding the high-rank updates.
\begin{mdframed}[
    backgroundcolor=yellow!20,
    linecolor=Goldenrod,
    roundcorner=4pt,
    linewidth=0.8pt,
    innerleftmargin=10pt, innerrightmargin=10pt,
    innertopmargin=10pt, innerbottommargin=10pt,
    skipabove=6pt, skipbelow=6pt
  ]
  \begin{lemma}
    The upper bound for the rank of the update $\Delta \mathrm{W} = \mathrm{A}(\mathrm{W} + \mathrm{U}\mathrm{V})\mathrm{B} - \mathrm{W}$ is $2R + r$, where $R = \operatorname{rank}(\mathrm{W})$ and $r = \operatorname{rank}(\mathrm{U}\mathrm{V})$.
  \end{lemma}
\end{mdframed}

The proof is in the appendix \autoref{lem:high_rank_proof}.

\subsection{Mechanism of Learning}
\label{subsec:mechanism}

BPT combines two complementary mechanisms: low-rank exploration and diagonal reweighting. Let the frozen base matrix be decomposed into rank-one directions,
\[
  \mathrm{W}_0 = \sum_{i=1}^{R} \mathbf{a}_i \mathbf{b}_i^\top .
\]
The diagonal matrices $\trainable{\mathrm{A}}$ and $\trainable{\mathrm{B}}$ rescale these existing directions:
\[
  \trainable{\mathrm{A}}\mathrm{W}_0\trainable{\mathrm{B}}
  =
  \sum_{i=1}^{R}
  \trainable{\mathrm{A}}\mathbf{a}_i \mathbf{b}_i^\top \trainable{\mathrm{B}} .
\]
However, diagonal scaling alone cannot create new directions. The low-rank term $\trainable{\mathrm{U}}\trainable{\mathrm{V}}$ provides this exploration:
\[
  \trainable{\mathrm{A}}
  (\mathrm{W}_0+\trainable{\mathrm{U}}\trainable{\mathrm{V}})
  \trainable{\mathrm{B}}
  =
  \sum_{i=1}^{R}
  \trainable{\mathrm{A}}\mathbf{a}_i \mathbf{b}_i^\top \trainable{\mathrm{B}}
  +
  \sum_{j=1}^{r}
  \trainable{\mathrm{A}}\mathbf{u}_j \mathbf{v}_j^\top \trainable{\mathrm{B}} .
\]
Thus, $\trainable{\mathrm{U}}$ and $\trainable{\mathrm{V}}$ discover new low-rank directions, while $\trainable{\mathrm{A}}$ and $\trainable{\mathrm{B}}$ immediately modulate both the original and newly discovered directions. Periodically, BPT merges the learned low-rank update into the base matrix,
\begin{equation}
  \label{eq:merge}
  \mathrm{W}_{t+1} := \mathrm{W}_t + \trainable{\mathrm{U}}_t\trainable{\mathrm{V}}_t ,
\end{equation}
and reinitializes $\trainable{\mathrm{U}}$ and $\trainable{\mathrm{V}}$ to explore new subspaces. Since the diagonal matrices are not merged or reinitialized, they continue to rescale all accumulated directions. After multiple merge cycles, BPT therefore compounds learning: low-rank factors add new directions to $\mathrm{W}$, while the diagonals continuously scale both old and new directions. This differs from ReLoRA-style training, where each update remains confined to a rank-$r$ subspace between merge steps. In BPT, the current low-rank factors explore new directions while the persistent diagonal matrices simultaneously reshape the accumulated base, allowing high-rank behavior at each step.

\subsection{Initialization and Merging}
\label{sec:init}
\paragraph{Initialization}
We initialize the fixed base weight $\mathrm{W}$ with Kaiming initialization \cite{he2015delvingdeeprectifierssurpassing}. The diagonal matrices $\mathrm{A}$ and $\mathrm{B}$ are initialized to one, so they initially act as identities. For the low-rank factors, $\mathrm{U}$ is initialized semi-orthogonally and $\mathrm{V}$ is initialized to zero. Thus $\mathrm{UV}=0$, and the initial forward pass matches the Kaiming-initialized base model.

\paragraph{Merging}

During training, BPT periodically adds the learned low-rank update into the frozen base weight using \autoref{eq:merge}. This preserves the directions already discovered while freeing the low-rank factors to explore a new subspace. After each merge, we reset $\mathrm U$ and $\mathrm V$ as
$$
\mathrm U := \text{init.orthogonal} \quad , \quad \mathrm V := \text{init.zero}
$$
Because $\mathrm{UV}=0$ immediately after reset, the effective matrix is unchanged at the merge point. The fresh semi-orthogonal $\mathrm U$ provides rank-$r$ directions for the next exploration phase. We do not merge the diagonal matrices, since doing so would not create new subspace directions.

Reinitializing the low-rank factors makes their optimizer exponential moving averages (EMAs) stale, whereas the diagonal EMAs remain valid because the diagonals are not reset. This helps stabilize training and avoids sudden loss spikes after merging.

\begin{figure*}[h]
  \centering
  \includegraphics[width=\linewidth]{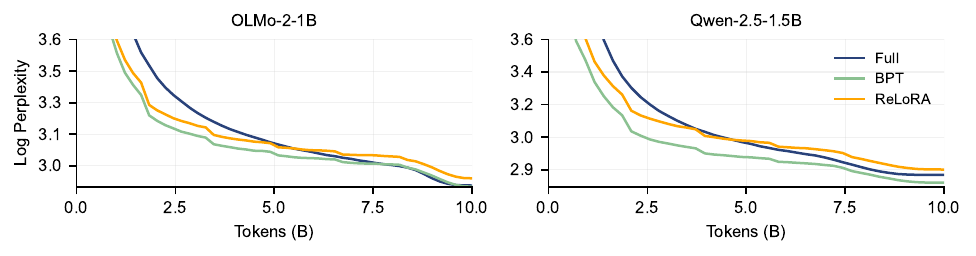}
  \caption{Eval loss over training steps. Left: OLMo-2-1B. Right: Qwen-2.5-1.5B. BPT tracks the Full baseline closely and outperforms ReLoRA over time, despite using substantially fewer trainable parameters.}
  \label{fig:1.5B_loss_qwen_olmo}
\end{figure*}

\subsection{Parameter and Compute Efficiency}

The primary advantage of BPT lies in its significant reduction in the number of trainable parameters compared to training the full weight matrix $\mathrm{W}$. For a single layer with weight matrix $\mathrm{W} \in \mathbb{R}^{n \times m}$, the trainable parameters in BPT are located only in the diagonal matrices $\mathrm{A}$ and $\mathrm{B}$, and the low-rank matrices $\mathrm{U}$ and $\mathrm{V}$. The total number of trainable parameters for one such layer in BPT is $n + m$ (from diagonal $\mathrm{A}$ and $\mathrm{B}$) plus $n \times r + r \times m$ (from low-rank $\mathrm{U}$ and $\mathrm{V}$). In contrast, training the full weight matrix requires $n \times m$ parameters. Since $r \ll \min(n, m)$, the total trainable parameters in BPT are substantially less than $nm$.
$$
\underbrace{n + m}_{\text{Diagonal Matrices}} + \underbrace{nr + rm}_{\text{Low-Rank Matrices}} \ll \underbrace{nm}_{\text{Full Matrix}}
$$
This parameter reduction directly translates into lower memory overhead during training with efficient kernel implementation. This trainable parameter count reduces the fixed memory cost associated with training neural networks as stated in \autoref{sec:intro}. Additionally, BPT requires approximately two-thirds of the linear-layer FLOPs of full training; we provide the full FLOP analysis in \autoref{sec:flop_analysis}.

\section{Related Work}

While the past few years have witnessed remarkable progress in reducing the memory footprint for fine-tuning large language models, memory-efficient pre-training remains considerably less explored. GaLore~\cite{zhao2024galorememoryefficientllmtraining} addresses this gap by proposing a novel approach to parameter-efficient training that projects gradients into a low-rank subspace, diverging from traditional methods that directly parameterize low-rank weight matrices. However, GaLore's reliance on Singular Value Decomposition (SVD) to identify the optimal low-rank approximation presents a significant computational bottleneck. The computational cost of SVD scales poorly with the matrix's dimensions, since SVD has $O(n^3)$ time complexity, making GaLore difficult to scale.

LoRA-the-explorer (LTE) \cite{huh2024trainingneuralnetworksscratch} extends LoRA to pre-training by exploring the low-rank solution space in parallel across different computing nodes. While LTE's parallel search strategy presents a seemingly promising avenue to mitigate the computational bottlenecks associated with full-rank pre-training, it unfortunately inherits the same fundamental drawback: at each step, optimization is limited to a low-rank subspace similar to other low-rank methods \autoref{background}.

Other LoRA variants further improve parameter-efficient adaptation but remain primarily designed for fine-tuning pretrained models. DoRA \cite{liu2024doraweightdecomposedlowrankadaptation} decomposes each pretrained weight into magnitude and direction components, using low-rank weights for directional updates while learning the magnitude separately. VeRA \cite{kopiczko2024veravectorbasedrandommatrix} instead freezes shared random low-rank matrices and trains only lightweight scaling vectors, greatly reducing adapter parameters. However, both approaches still rely on constrained low-rank directions or fixed random projections, and therefore do not directly address the need for repeatedly expanding high-rank update directions during from-scratch pre-training.

\begin{table*}[t]
  \centering
  \scriptsize
  \addtolength{\tabcolsep}{-1.8pt}
  \resizebox{\textwidth}{!}{
  \begin{tabular}{l|l|r|cc|ccccccccc|c|c}
    \toprule
    Model & Method & Trainable
    & Wiki & LAMB & LAMB
    & Hella & WinoG & PIQA & BoolQ
    & SciQ & OBQA & ARC-E & ARC-C
    & \textbf{Avg} & SWDE \\
    & & Params
    & ppl $\downarrow$ & ppl $\downarrow$ & acc $\uparrow$
    & acc $\uparrow$ & acc $\uparrow$ & acc $\uparrow$ & acc $\uparrow$
    & acc $\uparrow$ & acc $\uparrow$ & acc $\uparrow$ & acc $\uparrow$
    & & acc $\uparrow$ \\
    \midrule
    \multicolumn{16}{c}{{10 Billion Tokens}}\\
    \midrule
    & Full   & 1.5B & 2.0 & 47.9 & 30.9 & 33.1 & 51.3 & 66.9 & 59.9 & 81.1 & 23.6 & 59.2 & 23.0 & 47.7 & 17.7 \\
    Qwen-2.5-1.5B & ReLoRA & 0.5B & 5.0 & 42.8 & 32.2 & 33.1 & 50.2 & 66.5 & 60.6 & 82.6 & 23.2 & 58.8 & 24.5 & 48.0 & 15.0 \\
    \rowcolor{lightblue}
    & BPT    & 0.5B & 4.8 & 34.6 & 34.6 & 34.7 & 52.1 & 66.8 & 57.2 & 81.2 & 23.0 & 61.9 & 27.3 & 48.8 & 15.3 \\
    \midrule
    & Full   & 1.5B & 1.9 & 48.2 & 30.6 & 32.8 & 51.1 & 66.6 & 57.9 & 82.8 & 23.2 & 58.9 & 23.5 & 47.5 & 14.4 \\
    OLMo-2-1B & ReLoRA & 0.6B & 2.0 & 47.9 & 30.9 & 32.8 & 52.2 & 66.5 & 57.9 & 83.7 & 23.4 & 58.8 & 25.4 & 48.0 & 16.5 \\
    \rowcolor{lightblue}
    & BPT    & 0.6B & 2.0 & 40.7 & 33.0 & 33.1 & 52.6 & 67.7 & 60.7 & 82.2 & 23.6 & 59.9 & 25.7 & 48.7 & 19.5 \\
    \midrule
    \multicolumn{16}{c}{{100 Billion Tokens}}\\
    \midrule
    & Full & 1.5B & 3.4 & 18.1 & 42.8 & 41.4 & 53.5 & 71.6 & 59.7 & 87.5 & 28.8 & 69.2 & 34.0 & 54.3 & 23.9 \\
    Qwen-2.5-1.5B  & ReLoRA & 0.5B & 3.9 & 23.5 & 38.7 & 37.6 & 53.3 & 68.4 & 58.6 & 84.4 & 25.4 & 65.2 & 30.6 & 51.3 & 21.4 \\
    \rowcolor{lightblue}
    & BPT & 0.5B & 3.7 & 21.4 & 39.2 & 39.2 & 53.9 & 70.6 & 62.0 & 87.5 & 26.2 & 66.1 & 30.6 & 52.8 & 19.7 \\
    \bottomrule
  \end{tabular}
  }
  \caption{Performance across Language Modeling, Commonsense, QA, and Extraction Tasks: lower is better for perplexity; higher is better for accuracy.}
  \label{tab:all_metrics}
\end{table*}

\section{Empirical Experiments}

To evaluate the effectiveness of our proposed method, we conducted a series of pre-training experiments on language models of varying scales and assessed their performance on a diverse suite of downstream tasks. Our primary goal was to demonstrate that BPT can come close to the performance of standard full-parameter pre-training while requiring significantly fewer trainable parameters, thereby lowering the computational barrier to developing large foundation models.

All models were pre-trained from \textbf{scratch} using the standard next-token prediction objective. We used a 10 and 100 billion token subset of the FineWeb-Edu dataset \cite{penedo2024finewebdatasetsdecantingweb}. Our baselines for comparison are standard full-parameter pre-training (referred to as ``Full''), where all model parameters are trainable, and ReLoRA. Importantly, all methods use the same underlying model architecture and the same total number of parameters; the difference lies solely in the number of trainable parameters during training. We train three scales of models based on the Qwen2.5 architecture: 100M, 0.5B, and 1.5B parameters. Additionally, we also pre-train OLMo-2-1B \cite{olmo20242olmo2furious} on the 10B tokens of the same dataset. The detailed specifications of the model architectures are provided in \autoref{sec:model_architecture}. We also summarize the parameter reduction achieved through BPT in \autoref{tab:param_reduction}. For more details regarding hyperparameters and layers targeted, refer to \autoref{sec:hparams_and_setup}. As the model scales, BPT uses a smaller fraction of trainable parameters while remaining close to full-parameter performance, since layers like the embedding layer do not have an outsized effect on larger models

\begin{figure}[h]
  \centering
  \includegraphics[width=1\linewidth]{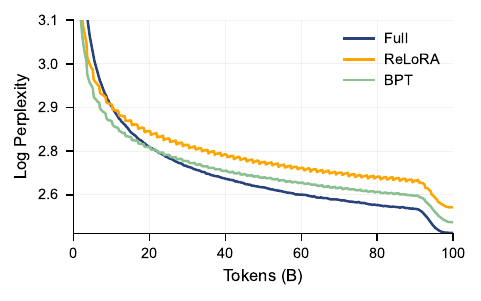}
  \caption{Qwen-2.5-1.5B: Validation Loss vs. Training Steps (Full vs. BPT) for 100B tokens}
\end{figure}

\subsection{Pre-training Performance}

Our pre-training experiments at 10 billion tokens demonstrate BPT achieves validation perplexity close to full-parameter training across different model scales while achieving substantial reductions in trainable parameters. \autoref{fig:log_log_scale} summarizes our experimental results across different model sizes. The x-axis represents the trainable parameters, not including embedding parameters. We show that, for a given model size, BPT achieves log-perplexity performance on a validation set similar to that of the full parameterized model, with significantly fewer trainable parameters.

Figure \ref{fig:1.5B_loss_qwen_olmo} compares the evaluation log perplexity on the validation set for the Qwen-2.5-1.5B  and OLMo-2 between the Full, ReLoRA, and BPT. We show that even with \textbf{66\% less trainable parameters}, BPT achieves validation loss close to full-parameter training. The log-perplexity of validation data shows that the model does not simply overfit the training data and has a similar evaluation loss profile as the baseline. Both BPT and ReLoRA have almost an identical number of trainable parameters. ReLoRA's eval loss slowly diverges from BPT's, even though both methods use a similar high learning rate.

\begin{figure*}[h!]
  \centering
  \includegraphics[width=\linewidth]{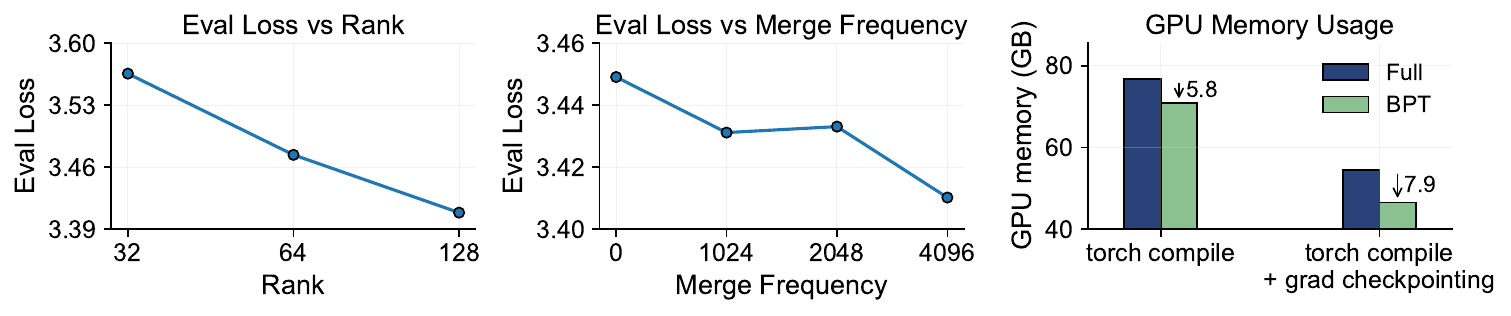}
  \caption{BPT ablations and memory summary. Left/Middle: rank and merge-frequency ablations on the 100M model trained on 10B tokens. Higher rank for the low-rank matrices leads to lower validation loss. Less frequent merges also lead to lower validation loss. Right: Peak GPU memory for Qwen-2.5-1.5B. BPT saves 5.8 GB to train a 1.5B model.}
  \label{fig:ablation_gpu}
\end{figure*}

\paragraph{Longer Training (100 Billion Tokens)}

To further assess the effectiveness of BPT, we pre-train Qwen-2.5-1.5B model on 100 billion tokens from FineWeb-Edu using the same setup as our main experiments. The final evaluation difference between BPT and Full is \textbf{0.0271 log perplexity} with 66\% fewer trainable parameters, compared with a 0.0656 gap between ReLoRA and Full.

\subsection{Downstream Evaluation}
The training loss and evaluation loss indicate strong performance for our proposed methods. To further evaluate whether the parameter efficiency of BPT impacts the model's ability to generalize to unseen tasks, we perform downstream evaluations on our Qwen-2.5-1.5B and OLMo-2-1B models. For all our evaluations, we do not further finetune any additional dataset and report zero-shot performance. Evaluations were performed using the LM evaluation harness \cite{eval_harness}. Following prior work at a similar model and data scale \cite{allenzhu2025physicslanguagemodels41, yang2025parallelizinglineartransformersdelta}, we evaluate commonsense and reasoning benchmarks including HellaSwag \cite{zellers2019hellaswagmachinereallyfinish}, WinoGrande \cite{sakaguchi2019winograndeadversarialwinogradschema}, PIQA \cite{bisk2019piqareasoningphysicalcommonsense}, BoolQ \cite{clark2019boolqexploringsurprisingdifficulty}, SciQ \cite{welbl2017crowdsourcingmultiplechoicescience}, OpenBookQA \cite{mihaylov2018suitarmorconductelectricity}, and ARC-Easy/Challenge \cite{clark2018thinksolvedquestionanswering}. We also report perplexity on held-out text using WikiText\cite{merity2016pointersentinelmixturemodels} and LAMBADA \cite{paperno2016lambadadatasetwordprediction}, and evaluate generation using SWDE \cite{lockard-etal-2019-openceres}. For Wikitext, we report byte perplexity.
\autoref{tab:all_metrics} summarizes our results. At 10B tokens, BPT is competitive with Full on average in our benchmark suite, while at 100B tokens Full performs better on average. This suggests that BPT preserves much, but not all, downstream performance under reduced trainable-parameter training.

\subsection{Comparison of PEFT Methods}

\begin{table}[h!]
  \centering
  \caption{Comparison of PEFT Methods: Trainable Parameter and Eval Loss. Models are pre-trained from scratch on 10 billion tokens.}
  \label{tab:peft_base}
  \vspace{-2pt}
  \begin{tabular}{l|r|l}
    \toprule
    Method & Trainable & Eval \\
    & Param     & Loss $\downarrow$ \\
    \midrule
    HyperAdapt & 137M & 5.486 \\
    VeRA       & 137M & 4.640 \\
    LoRA       & 206M & 3.051 \\
    DoRA       & 206M & 3.102 \\
    ReLoRA     & 206M & 3.038 \\
    \rowcolor{lightblue}
    BPT        & 206M & 3.014 \\
    \midrule
    Full       & 494M & 3.022 \\
    \bottomrule
  \end{tabular}
\end{table}

Since BPT builds on prior works such as LoRA, ReLoRA, and HyperAdapt, we can evaluate the performance of each of these components by directly pre-training using these methods and comparing against BPT. For this ablation study, we pre-train the Qwen-2.5-0.5B model in the same setup as our main experiments with 10 billion tokens and report their evaluation loss along with the trainable parameters. This study also allows us to see the effects of only training low-rank matrices $\mathrm{U}\mathrm{V}$ without merging (LoRA), the effect of training low-rank matrices $\mathrm{U}\mathrm{V}$ with merging (ReLoRA), and the effect of training only the diagonal matrices $\mathrm{A}$ and $\mathrm{B}$ (HyperAdapt).

The result is summarized in \autoref{tab:peft_base}. HyperAdapt, which only uses diagonal matrices to update $\mathrm{W}$, performs poorly since the base weight matrix $\mathrm{W}$ does not contain any relevant subspace that can be adjusted, since $\mathrm{W}$ at initialization is just a random matrix. LoRA without periodic rank accumulation performs better than HyperAdapt, and the difference between LoRA's evaluation loss and ReLoRA's evaluation loss is only 0.01. BPT and Full both have evaluation losses that are similar compared to other methods. Additionally, to exactly match the number of trainable parameters in ReLoRA and BPT, we conducted an additional experiment using ReLoRA with rank (r+1=129). Under the same training setup, rank-128 ReLoRA achieved a final evaluation loss of 3.038, while rank-129 ReLoRA achieved 3.033. In comparison, rank-128 BPT achieved a lower evaluation loss of 3.014. Thus, increasing the ReLoRA rank to obtain an exact parameter-matched comparison yields only a small improvement and does not account for the performance advantage of BPT.

\subsection{Rank and Merge Frequency}

To better understand the role of rank and merge frequency in BPT, we ablate both under the same setting as our main results, using the 100M model trained on 10B tokens. As shown in \autoref{fig:ablation_gpu}, increasing rank improves evaluation loss with diminishing returns, while less frequent merges generally perform better by allowing Adam to estimate its moments more stably between resets. Detailed GPU memory measurements are provided in Appendix~\ref{sec:gpu_memory_footprint}.

\section{Conclusion}
In this paper, we introduced Bi-Phase Training (BPT), a parameter-efficient pre-training method that combines diagonal high-rank transformations with periodically merged low-rank adapters. Across language-model pre-training from 100M to 1.5B parameters, BPT approaches full-training validation loss and preserves much of its downstream performance, while optimizing substantially fewer trainable parameters. This suggests that BPT can be a practical route to more efficient foundation-model pre-training.

\section*{Limitations}

Our experiments are limited to autoregressive language-model pre-training models up to 1.5B parameters. Future work should test BPT on larger models, longer training runs, other architectures and also different modalities such as diffusion.

\section*{Acknowledgments}
This work used the DeltaAI system at the National Center for Supercomputing Applications through allocation CIS250267 from the Advanced Cyberinfrastructure Coordination Ecosystem: Services and Support (ACCESS) program, which is supported by National Science Foundation grants 2138259, 2138286, 2138307, 2137603, and 2138296.

\bibliography{custom}

\appendix
\onecolumn
\section{Momentum and Adam's Update Rank}
\label{sec:adams_update_rank}

For SGD with momentum, the update rank at time $t$ is upper-bounded by $tB$ since momentum is a linear combination of past gradients. In modern optimizers such as Adam \cite{kingma2017adammethodstochasticoptimization}, where first and second exponential moving averages (EMAs) are stored and transformed, the rank of the update is no longer bounded just by the batch size but depends on the rank of its first and second EMAs. Here we derive a loose upper-bound for Adam.

For Adam, let $G_t \in \mathbb{R}^{n\times m}$ denote the mini-batch gradient at step $t$ (for batch size $B$, we have $\operatorname{rank}(G_t) \le \min(n,m,B)$ as above). Adam maintains exponential moving averages of the first and second moments:
\begin{align*}
  M_t &= \beta_1 M_{t-1} + (1-\beta_1)G_t, \\
  V_t &= \beta_2 V_{t-1} + (1-\beta_2)(G_t \odot G_t),
\end{align*}
where $\odot$ denotes the element-wise product. The parameter update is
$$
\Delta W_t = -\alpha  \frac{M_t}{\sqrt{V_t}+\varepsilon} =  -\alpha  M_t \odot  \frac{1}{\sqrt{V_t}+\varepsilon}
$$
where the division and square-root are applied element-wise. Using Hadamard product

Since $M_t$ is a linear combination of past gradients,
\[
  M_t = (1-\beta_1)\sum_{k=1}^t \beta_1^{t-k} G_k,
\]
and using $\operatorname{rank}(\sum_{k=1}^{t} G_k)\le \sum_{k=1}^{t} \operatorname{rank}(G_k)$, we obtain an upper-bound
\[
  \operatorname{rank}(M_t)
  \le \sum_{k=1}^t \operatorname{rank}(G_k)
  \le tB,
\]
hence $\operatorname{rank}(M_t)\le \min(n,m,tB)$. Using the Hadamard-rank inequality: for any $A,B\in\mathbb{R}^{n\times m}$,
\[
  \operatorname{rank}(A\odot B) \le \operatorname{rank}(A)\,\operatorname{rank}(B).
\]
Applying this to $\Delta W_t$ yields
\[
  \operatorname{rank}(\Delta W_t)
  = \operatorname{rank}(M_t\odot \frac{1}{\sqrt{V_t}+\varepsilon})
  \le \operatorname{rank}(M_t)\,\operatorname{rank}(\frac{1}{\sqrt{V_t}+\varepsilon})
\]
Hence, for Adam, the rank of its $\Delta W_t$ is upper-bounded by the product of its first and second moments' ranks. Here, we do not provide the rank of the second moment primarily because performing element-wise square-root on a matrix leads to full rank in some cases, even for rank-one matrices.

\section{BPT Rank Proof}
\label{lem:high_rank_proof}

\begin{lemma}
  The upper bound for the rank of the update $\Delta \mathrm{W} = \mathrm{A}(\mathrm{W} + \mathrm{U}\mathrm{V})\mathrm{B} - \mathrm{W}$ is $2R + r$, where $R = \operatorname{rank}(\mathrm{W})$ and $r = \operatorname{rank}(\mathrm{U}\mathrm{V})$.
\end{lemma}

\begin{proof}
  We have
  \[
    \Delta \mathrm{W}
    = \mathrm{A}(\mathrm{W} + \mathrm{U}\mathrm{V})\mathrm{B} - \mathrm{W}
    = \mathrm{A}\mathrm{W}\mathrm{B} + \mathrm{A}\mathrm{U}\mathrm{V}\mathrm{B} - \mathrm{W}.
  \]
  For all conformable matrices $\mathrm{X}$ and $\mathrm{Y}$,
  $\operatorname{rank}(\mathrm{X} + \mathrm{Y})
  \le \operatorname{rank}(\mathrm{X}) + \operatorname{rank}(\mathrm{Y})$.
  Therefore,
  \[
    \operatorname{rank}\!\bigl(\mathrm{A}\mathrm{W}\mathrm{B} + \mathrm{A}\mathrm{U}\mathrm{V}\mathrm{B} - \mathrm{W}\bigr)
    \le \operatorname{rank}(\mathrm{A}\mathrm{W}\mathrm{B})
    + \operatorname{rank}(\mathrm{A}\mathrm{U}\mathrm{V}\mathrm{B})
    + \operatorname{rank}(-\mathrm{W}).
  \]
  Since matrix multiplication cannot increase rank,
  $\operatorname{rank}(\mathrm{A}\mathrm{X}\mathrm{B}) \le \operatorname{rank}(\mathrm{X})$ for any $\mathrm{X}$, and hence
  \[
    \operatorname{rank}\!\bigl(\mathrm{A}\mathrm{W}\mathrm{B} + \mathrm{A}\mathrm{U}\mathrm{V}\mathrm{B} - \mathrm{W}\bigr)
    \le \operatorname{rank}(\mathrm{W})
    + \operatorname{rank}(\mathrm{U}\mathrm{V})
    + \operatorname{rank}(-\mathrm{W}).
  \]
  We have $\operatorname{rank}(\mathrm{U}\mathrm{V}) = r$ and
  $\operatorname{rank}(\mathrm{W}) = \operatorname{rank}(-\mathrm{W}) = R$, so
  \[
    \operatorname{rank}\!\bigl(\Delta \mathrm{W}\bigr)
    \le R + r + R = 2R + r.
  \]
  Hence the rank of the update $\Delta \mathrm{W}$ is upper-bounded by $2R + r$. Trivially, it is also bounded by its dimensions so $\operatorname{rank} (\Delta \mathrm{W}) \leq \min(n,m, 2R+r)$
\end{proof}

\section{FLOP Analysis}
\label{sec:flop_analysis}
For FLOP analysis, we use the same setup as \cite{schulman2025lora}. In full training, for a weight matrix
$\mathrm{W} \in \mathbb{R}^{N \times N}$, an input $x \in \mathbb{R}^N$, and an output
$y \in \mathbb{R}^N$, we have
\begin{align*}
  \text{Forward:}\qquad
  y &= \mathrm{W}x
  && \left(N^2 \text{ multiply-adds}\right),
  \\[0.5em]
  \text{Backward:}\qquad
  \bar{x} &= \mathrm{W}^\top \bar{y}
  && \left(N^2 \text{ multiply-adds}\right),
  \\
  \bar{\mathrm{W}} &\mathrel{+}= \bar{y}x^\top
  && \left(N^2 \text{ multiply-adds}\right).
\end{align*}

Here, $\bar{\mathrm{W}}$, $\bar{y}$, and $\bar{x}$ denote gradients of the loss with
respect to $\mathrm{W}$, $y$, and $x$, respectively. Combining the forward pass
$\left(N^2\right)$ and backward pass $\left(2N^2\right)$ gives
\[
  \mathrm{FLOPs}_{\mathrm{Full}} = 3N^2.
\]

In BPT, we replace $\mathrm{W}$ by
$\mathrm{A}(\mathrm{W} + \mathrm{U}\mathrm{V})\mathrm{B}$, where
$\mathrm{A}, \mathrm{B} \in \mathbb{R}^{N \times N}$ are diagonal matrices and
$\mathrm{U} \in \mathbb{R}^{N \times R}$,
$\mathrm{V} \in \mathbb{R}^{R \times N}$ are low-rank matrices with $R \ll N$.
Since we only update $\bar{\mathrm{A}}$, $\bar{\mathrm{B}}$,
$\bar{\mathrm{U}}$, and $\bar{\mathrm{V}}$, the update of $\bar{\mathrm{W}}$
is replaced by cheaper operations.

The FLOP estimate for BPT is therefore
\begin{align*}
  \text{Forward:}\qquad
  y &= \mathrm{A}(\mathrm{W} + \mathrm{U}\mathrm{V})\mathrm{B}x
  && \left(N^2 + 2NR + 2N \text{ multiply-adds}\right),
  \\[0.75em]
  \text{Backward:}\qquad
  \bar{x}
  &= \mathrm{B}^\top(\mathrm{W} + \mathrm{U}\mathrm{V})^\top
  \mathrm{A}^\top \bar{y}
  && \left(N^2 + 2NR + 2N \text{ multiply-adds}\right),
  \\
  \bar{\mathrm{U}}
  &\mathrel{+}= (\mathrm{A}^\top \bar{y})(\mathrm{V}\mathrm{B}x)^\top
  && \left(NR \text{ multiply-adds}\right),
  \\
  \bar{\mathrm{V}}
  &\mathrel{+}= (\mathrm{U}^\top \mathrm{A}^\top \bar{y})(\mathrm{B}x)^\top
  && \left(NR \text{ multiply-adds}\right),
  \\
  \bar{\mathrm{A}}
  &\mathrel{+}= \bar{y} \odot
  \bigl((\mathrm{W} + \mathrm{U}\mathrm{V})\mathrm{B}x\bigr)
  && \left(N \text{ multiply-adds}\right),
  \\
  \bar{\mathrm{B}}
  &\mathrel{+}=
  \bigl((\mathrm{W} + \mathrm{U}\mathrm{V})^\top
  \mathrm{A}^\top \bar{y}\bigr) \odot x
  && \left(N \text{ multiply-adds}\right).
\end{align*}

Thus, combining the BPT forward pass
$\left(N^2 + 2NR + 2N\right)$ and backward pass
$\left(N^2 + 4NR + 4N\right)$ gives
\[
  \mathrm{FLOPs}_{\mathrm{BPT}}
  =
  2N^2 + 6NR + 6N.
\]

This is close to the FLOP estimate for LoRA \cite{schulman2025lora},
which is
\[
  \mathrm{FLOPs}_{\mathrm{LoRA}}
  =
  2N^2 + 6NR.
\]
BPT incurs an additional $6N$ cost due to the diagonal factors. Compared to full training, which costs $3N^2$ per layer, BPT uses
\begin{align*}
  \frac{\mathrm{FLOPs}_{\mathrm{BPT}}}
  {\mathrm{FLOPs}_{\mathrm{Full}}}
  &=
  \frac{2N^2 + 6NR + 6N}{3N^2}
  =
  \frac{2}{3}
  +
  \frac{2R}{N}
  +
  \frac{2}{N}.
\end{align*}

When $R \ll N$, this is only slightly larger than $2/3$, meaning that BPT uses roughly two-thirds of the compute of full training. This analysis follows the same setup as \cite{schulman2025lora} and considers only FLOPs in linear layers, which account for the majority of FLOPs during training.

\section{Model Architecture and Parameter Reduction}
\label{sec:model_architecture}

All models are based on the Qwen2.5 architecture \cite{qwen2025qwen25technicalreport}: in addition to the standard 0.5B and 1.5B parameter versions, we developed a custom 100M parameter variant by scaling down the same architecture. We also pre-train OLMo-2-1B \cite{olmo20242olmo2furious}. Although both models are transformer-based auto-regressive language models, they differ slightly in their attention implementation (Grouped Query Attention vs Multi-head Attention). Additionally, Qwen-2.5 ties its input and output embeddings, whereas OLMo-2 has separate input and output embeddings. For OLMo-2, we directly use the official implementation from the authors.

\begin{table}[h!]
  \centering
  \caption{Model Architecture Detail and Total Training Tokens}
  \begin{tabular}{lccccc}
    \toprule
    \textbf{Models} & \textbf{Layers} & \textbf{Heads (Q / KV)} & \textbf{Hidden Size} & \textbf{Intermediate Size} & \textbf{Training Tokens} \\
    \midrule
    100M   & 8  & 8/ 2  & 512   & 2048  & 10B   \\
    0.5B  & 24 & 14 / 2 & 896   & 4864  & 10B   \\
    1.5B    & 36 & 16 / 2 & 2048  & 11008 & 10B and 100B  \\
    \bottomrule
  \end{tabular}
  \label{tab:qwen2.5_models}
\end{table}

\begin{table*}
  \centering
  \caption{Trainable parameter counts by method and model size}
  \begin{tabular}{l|rr |rr}
    \toprule
    \textbf{Model} & \textbf{Full} & \textbf{BPT} & \textbf{Total} & \textbf{Non-embedding} \\
    & \textbf{Trainable} & \textbf{Trainable} & \textbf{Reduction (\%)} & \textbf{Reduction (\%)} \\
    \midrule
    Small   & 108M & 89M & 17.8 & 63.5 \\
    Medium  & 0.5B & 0.2B & 58.2 & 80.4 \\
    Large   & 1.5B & 0.5B & 65.9 & 77.4 \\
    \bottomrule
  \end{tabular}
  \label{tab:param_reduction}
\end{table*}

\section{Hyperparameters and Experimental Setup}
\label{sec:hparams_and_setup}
For BPT and other baselines, we applied this method exclusively to the linear layers within the transformer blocks. Parameters such as layer normalization, embedding layers, and (lm\_head) were kept fully trainable. For the Qwen2.5 architecture, attention matrices, such as the query projection layer \texttt{q\_head}, have dimensions 518 by 128. Applying BPT to such layers would increase the number of trainable parameters if the low-rank matrices' rank is set to 128, undermining the parameter efficiency goal for these specific matrices. Hence, we leave them unchanged.
We did not conduct extensive hyperparameter search for learning rate. Following the trend of PEFT papers, where learning rate is higher compared to full fine-tuning. We simply used 10x of the base learning rate which was 2e-4 for all our base models. For ReLoRA with Qwen 1.5B model and LoRA with Qwen 0.5B model, we used a learning rate of 1e-3 after multiple crashes with learning rate 2e-3. Additionally, all the parameters in the model were trained using bf16 including the low-rank matrices with the exception of the diagonal matrices which were kept in fp32 for stability. For a fair comparison, all experiments for a given model size used the same batch size and were trained for the same total number of steps. The optimizer used was AdamW \cite{loshchilov2019decoupledweightdecayregularization}, and the learning rate followed a warm-up and cosine decay schedule (WSD).

\begin{table}[h!]
  \footnotesize
  \addtolength{\tabcolsep}{-1pt}
  \centering
  \caption{The hyperparameters used for pre-training Qwen-2.5 Models.}
  \begin{tabular}{ll|cccc}
    \hline
    \toprule
    Method  & Models     &  100M & 0.5B & 1.5B\\
    \midrule
    & Optimizer   & \multicolumn{3}{c}{AdamW} \\
    & Warmup Steps & \multicolumn{3}{c}{2000} \\
    & Decay Steps & \multicolumn{3}{c}{10000,8000,8000} \\
    & Max Grad Norm & \multicolumn{3}{c}{1.0} \\
    & Max Seq. Len & \multicolumn{3}{c}{512} \\
    & Batch Size & \multicolumn{3}{c}{320,512,1024} \\
    & LR Schedule & \multicolumn{3}{c}{WSD} \\
    & Tokens & \multicolumn{3}{c}{10B} \\
    \midrule
    \multirow{1}{*}{\makecell{Full}} \
    & Learning Rate & \multicolumn{3}{c}{2e-4} \\
    \midrule
    \multirow{1}{*}{\makecell{BPT}} \
    & Learning Rate & \multicolumn{3}{c}{2e-3} \\
    & Rank & \multicolumn{3}{c}{128,128,256} \\
    \midrule
    \multirow{1}{*}{\makecell{ReLoRA}} \
    & Learning Rate & \multicolumn{3}{c}{2e-3,1e-3} \\
    & Rank & \multicolumn{3}{c}{,128,256} \\
    \midrule
    \multirow{1}{*}{\makecell{LoRA}} \
    & Learning Rate & \multicolumn{3}{c}{1e-3} \\
    & Rank & \multicolumn{3}{c}{128} \\
    \midrule
    \multirow{1}{*}{\makecell{HyperAdapt}} \
    & Learning Rate & \multicolumn{3}{c}{2e-3} \\
    \bottomrule
  \end{tabular}
\end{table}
\begin{table}[h!]
  \footnotesize
  \addtolength{\tabcolsep}{-1pt}
  \centering
  \caption{The hyperparameters used for pre-training OLMo-2-1B.}
  \begin{tabular}{ll|cccc}
    \hline
    \toprule
    Method  & Models     &   & 1B & \\
    \midrule
    & Optimizer   & \multicolumn{3}{c}{AdamW} \\
    & Warmup Steps & \multicolumn{3}{c}{2000} \\
    & Decay Steps & \multicolumn{3}{c}{8000} \\
    & Max Grad Norm & \multicolumn{3}{c}{1.0} \\
    & Max Seq. Len & \multicolumn{3}{c}{512} \\
    & Batch Size & \multicolumn{3}{c}{960} \\
    & LR Schedule & \multicolumn{3}{c}{WSD} \\
    & Tokens & \multicolumn{3}{c}{10B} \\
    \midrule
    \multirow{1}{*}{\makecell{Full}} \
    & Learning Rate & \multicolumn{3}{c}{2e-4} \\
    \midrule
    \multirow{1}{*}{\makecell{BPT}} \
    & Learning Rate & \multicolumn{3}{c}{2e-3} \\
    & Rank & \multicolumn{3}{c}{256} \\
    \midrule
    \multirow{1}{*}{\makecell{ReLoRA}} \
    & Learning Rate & \multicolumn{3}{c}{2e-3} \\
    & Rank & \multicolumn{3}{c}{256} \\
    \bottomrule
  \end{tabular}
\end{table}

\begin{table}[h!]
  \footnotesize
  \addtolength{\tabcolsep}{-1pt}
  \centering
  \caption{The hyperparameters used for pre-training Qwen-2.5 1.5B on 100 billion tokens.}
  \begin{tabular}{ll|c}
    \hline
    \toprule
    Method  & Models     & 1.5B\\
    \midrule
    & Optimizer      & AdamW \\
    & Warmup Steps   & 2000 \\
    & Decay Steps    & 24,000  \\
    & Max Grad Norm  & 1.0 \\
    & Max Seq. Len   & 512 \\
    & Batch Size     & 1024 \\
    & LR Schedule    & WSD \\
    & Tokens         & 100B \\
    \midrule
    \multirow{1}{*}{\makecell{Full}}
    & Learning Rate  & 2e-4 \\
    \midrule
    \multirow{1}{*}{\makecell{BPT}}
    & Learning Rate  & 2e-3 \\
    & Rank           & 256 \\
    \bottomrule
  \end{tabular}
\end{table}

\section{GPU Memory Footprint}
\label{sec:gpu_memory_footprint}

On a 4-GPU NVIDIA GH200 setup, we evaluate the memory consumption of Qwen-2.5-1.5B at approximately 8192 tokens per GPU (16 microbatches $\times$ 512 sequence length). Compared to full training, BPT reduces peak GPU memory by 5.8 GB per GPU with \texttt{torch.compile} and by 7.9 GB per GPU with \texttt{torch.compile} plus gradient checkpointing, as summarized in \autoref{fig:ablation_gpu}. These memory savings come with roughly $7\%$ wall-clock overhead under \texttt{torch.compile}; because this uses a pure PyTorch implementation without custom kernels, the overhead is likely an upper bound.

\end{document}